\documentclass[lettersize,journal]{IEEEtran}
\usepackage{amsmath,amsfonts}
\usepackage[noend]{algorithmic}
\usepackage{algorithm}
\usepackage{array}
\usepackage[caption=false,labelfont=sf,textfont=sf]{subfig}
\usepackage{textcomp}
\usepackage{stfloats}
\usepackage{url}
\usepackage{slashbox}
\usepackage{booktabs}
\usepackage{xcolor}
\usepackage{verbatim}
\usepackage{graphicx}
\usepackage{cite}
 \usepackage{nameref}
\usepackage{graphicx}
\usepackage{amssymb}
\usepackage{mathtools}  
\usepackage{amsthm}
\usepackage{bbm}
\usepackage{tcolorbox}
\usepackage{pifont}
\newcommand{\cmark}{\ding{51}}%
\newcommand{\xmark}{\ding{55}}%
\newtheorem{theorem}{Theorem}

\newtheorem{definition}{Definition}

\usepackage{hyperref}
\begin{document}

\title{Regularization via $f$-Divergence: An Application to Multi-Oxide Spectroscopic Analysis}

\author{Weizhi Li,~\IEEEmembership{Member,~IEEE,}
        Natalie Klein,
        Brendan Gifford,
        Elizabeth Sklute,
        Carey Legett,
        Samuel Clegg
\thanks{W. Li and N. Klein are with Statistical Sciences group at Los Alamos National Laboratory, Los Alamos, New Mexico,
USA. (email: weizhili@lanl.gov; neklein@lanl.gov)}
\thanks{B. Gifford is with Physics and Chemistry of Materials at Los Alamos National Laboratory, Los Alamos, New Mexico,
USA.}
\thanks{E. Sklute is with Analytical Earth Science at Los Alamos National Laboratory, Los Alamos, New Mexico, USA.}

\thanks{C. Legett is with Space Remote Sensing and Data Science at Los Alamos National Laboratory, Los Alamos, New Mexico, USA.}
\thanks{S. Clegg is with Physical Chemistry \& Applied Spectroscopy at Los Alamos National Laboratory, Los Alamos, New Mexico, USA.}
 }



\maketitle

\begin{abstract}
In this paper, we address the task of characterizing the chemical composition of planetary surfaces using convolutional neural networks (CNNs). Specifically, we seek to predict the multi-oxide weights of rock samples based on spectroscopic data collected under Martian conditions. We frame this problem as a multi-target regression task and propose a \textit{novel regularization method} based on $f$-divergence. The $f$-divergence regularization is designed to constrain the distributional discrepancy between predictions and noisy targets. This regularizer serves a dual purpose: on the one hand, it mitigates overfitting by enforcing a constraint on the distributional difference between predictions and noisy targets. On the other hand, it acts as an auxiliary loss function, penalizing the neural network when the divergence between the predicted and target distributions becomes too large. To enable backpropagation during neural network training, we develop a differentiable $f$-divergence and incorporate it into the $f$-divergence regularization, making the network training feasible. We conduct experiments using spectra collected in a Mars-like environment by the remote-sensing instruments aboard the \textit{Curiosity} and \textit{Perseverance} rovers. Experimental results on multi-oxide weight prediction demonstrate that the proposed $f$-divergence regularization performs better than or comparable to standard regularization methods including $L_1$, $L_2$, and dropout. Notably, combining the $f$-divergence regularization with these standard regularization further enhances performance, outperforming each regularization method used independently.
\end{abstract}

\begin{IEEEkeywords}
Oxide-weights, multi-response regression, neural network, regularization, $f$-divergence
\end{IEEEkeywords}

\section{Introduction}
\IEEEPARstart{C}{haracterizing} the chemical composition of planetary surfaces is a fundamental task in planetary science. For instance, analyzing the oxide composition of rock samples from the Martian environment can provide insights into past or present biological activity on Mars~\cite{hassler2014mars}. Reflecting the importance of planetary surface exploration, NASA launched the Mars rover Perseverance as part of its Mars 2020 mission to investigate the astrobiological environment on Mars. Since its successful landing in 2021, the rover has actively sought astrobiological evidence in rock samples using its laser-induced breakdown spectroscopy (LIBS) system. The LIBS instrument emits a laser pulse at a wavelength of 1067 nm onto rock surfaces, creating a hot plasma of vaporized material. Analyzing the light emitted by the plasma reveals the chemical composition of the rocks. Due to the large number of elements present in complexity of the LIBS plasma and the resultant LIBS spectrum, linking the chemical composition to spectral feature positions and intensities was initially a difficult task.
With advancements in machine learning (ML), planetary scientists have developed multiple computational tools to advance and automate the spectroscopic analysis of rock samples, leveraging efficient and highly predictive ML models to better understand extraterrestrial environments. One classical ML solution for this task is the partial least squares (PLS) method based on linear regression, a multivariate analysis approach introduced in~\cite{wold1984collinearity}. The core concept of PLS is to identify two latent spaces that exhibit a high correlation between the feature and target variables. Subsequently, a linear regression model is developed to address the prediction problem within these latent spaces. Since its inception, PLS has become a widely used tool in chemometrics, achieving notable success in automating spectroscopic analyses~\cite{cheng2017partial,janik2009prediction,aznar2003prediction,song2023fractional, song2013using}. For predicting multi-oxide compositions in rocks, the authors of~\cite{clegg2009multivariate} developed an ensemble regression model based on PLS and published their prediction results in the Planetary Data System (PDS)~\cite{mcmahon1996overview} as a benchmark dataset.
\begin{figure*}[h!]
    \centering
\subfloat[Training MSE: 0.13 (\textbf{Proper divergence}) Test MSE: \textbf{0.17}]{\includegraphics[width=0.32\textwidth]{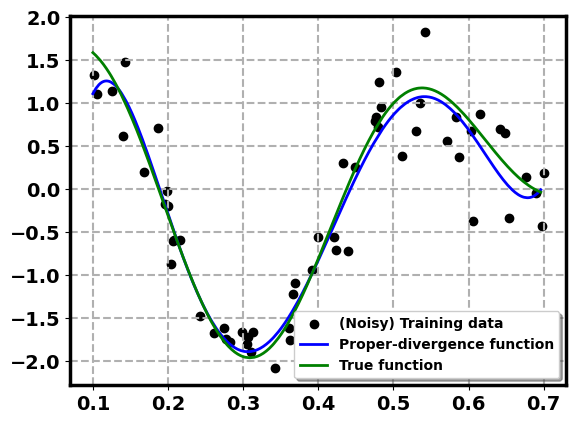}
\label{ProperDivergence}
}\hfill
\subfloat[Training MSE: 0.09 (Small divergence) Test MSE: 0.24]{\includegraphics[width=0.32\textwidth]{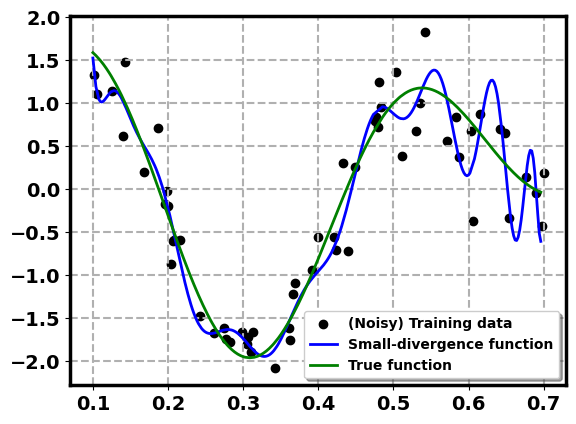}
\label{SmallDivergence}}\hfill
\subfloat[Training MSE: 0.88 (Large divergence) Test MSE: 1.02]{\includegraphics[width=0.32\textwidth]{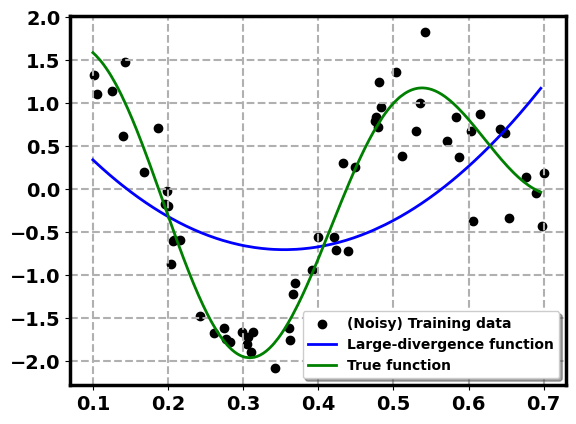}
\label{LargeDivergence}}
\caption{True function and various approximation functions are compared, along with their training and test mean squared errors (MSE). The training MSE \textit{implicitly} quantifies the divergence between the training targets and predictions. Consequently, the function approximation in (a), which maintains an appropriate level of divergence, achieves a smaller test error compared to the approximations in (b) and (c), where the divergences are too small and large, respectively.}
\label{FigVariousDivergence}
\end{figure*}
While PLS regression has demonstrated remarkable success in linking oxide compositions to LIBS spectral features, the LIBS spectral response of each material is often complicated by various matrix effects. As a result, PLS, being a linear model, may not be the most suitable machine learning approach for capturing the nonlinear relationships between LIBS spectra and the multi-oxide compositions found in rocks. Artificial intelligence (AI) methods, such as convolutional neural networks (CNNs), have gained widespread success in solving complex tasks, including face recognition, autonomous driving, and medical image analysis. These tasks are considered highly challenging, yet AI systems based on CNNs have achieved or nearly achieved human-level intelligence in these domains~\cite{lu2019review, cocskun2017face, anwar2018medical}. This success primarily stems from the nonlinear convolutional structure of neural networks and the explosion of available data along with the huge advance in the computational power. Recognizing the limitations of the linearity in PLS, scientists have developed multivariate analysis techniques using CNNs. Studies such as~\cite{acquarelli2017convolutional, shen2021automated, ghosh2019deep} showcase CNN-based machine learning models for spectroscopic analysis. However, the insufficiency of data is arguably the biggest obstacle preventing the seamless application of CNNs to spectroscopic analysis. To address this challenge, regularization techniques such as $L_2$ regularization are employed during neural network training to mitigate overfitting when working with small datasets.

In this paper, we focus on the \textit{multi-target regression} task, predicting multi-oxide weights using a CNN~\cite{zhang2015character}. We assume a data model where the measured multi-oxide weights contain noise caused by ``random" factors, such as variance in the oxide measurement process or the highly variable extraterrestrial environment. This necessitates the need for regularization during network training to prevent the network from overfitting to the noisy measurements of oxide weights. To address this need,  we propose a novel regularization method, \textit{$f$-divergence regularization}, for multi-target regression tasks. Figure~\ref{FigVariousDivergence} illustrates a true (or target) function alongside three polynomial approximations, with their respective mean squared errors (MSE) computed over the training and test datasets.  As shown, a high-degree polynomial (Figure~\ref{FigVariousDivergence} (b)) with a low training MSE tends to overfit the noisy training data,  while a low-degree polynomial (Figure~\ref{FigVariousDivergence} (c)) with a high training MSE  underfits. Consequently, a polynomial approximation (Figure~\ref{FigVariousDivergence}(a)) with an intermediate degree  and training MSE aligns most closely with the target function. This indicates that a good function approximation requires not only an appropriate model complexity but also a suitable divergence between the model predictions and the training targets. This insight motivates the explicit enforcement of divergence between prediction  and target distributions during neural network training. To this end, we propose the $f$-divergence regularization,   which imposes a distributional difference quantified as $f$-divergence~\cite{renyi1961measures} between the network predictions and the targets.  This prevents the network from overfitting to the distribution of noisy targets—specifically, the multi-oxide weights considered in this paper. Unlike conventional regularizers, such as $L_2$ regularization~\cite{tikhonov1963solution} or dropout~\cite{wager2013dropout}, which act on network parameters or intermediate layers, $f$-divergence regularization explicitly regularizes the network output through the $f$-divergence between the network prediction and target variables. This approach bears a resemblance to the popular label smoothing regularization~\cite{li2020regularization, szegedy2016rethinking}, which smooths a one-hot vector of categorical targets in classification tasks. However, by explicitly enforcing a distributional difference, the $f$-divergence regularization also penalizes the network when the divergence between the prediction and target distributions exceeds a pre-defined threshold,  acting as an auxiliary loss function to prevent the significant deviation from the target distribution. 

Focusing on an important technical aspect of $f$-divergence regularization, we construct a differentiable estimation of the $f$-divergence, enabling feasible network training with backpropagation~\cite{rumelhart1986learning}. By leveraging a specific form of $f$-divergence proposed in~\cite{berisha2015empirically}, we ensure that the differentiable $f$-divergence is bounded between 0 and 1, greatly reducing the searching space of regularization strengths across various tasks. We conduct experiments using spectroscopic data collected in a simulated Martian environment by the ChemCam and SuperCam laboratory instruments~\cite{clegg2017recalibration}, which are remote-sensing instrument suited on the Mars rover \textit{Curiosity} and \textit{Perseverance}. Experimental results show that training with the $f$-divergence regularization yields root mean squared errors (RMSEs) significantly smaller or comparable to those achieved by $L_1$, $L_2$ and dropout regularization. Notably, combining $f$-divergence regularization with $L_1$, $L_2$ or dropout further reduces RMSEs, outperforming the independent use of these standard regularization methods when applied independently.

There is a body of work~\cite{daunas2023empirical, cheng2020posterior, zhong2023learning} that applies $f$-divergence to ML problems by incorporating it into the loss function, serving various distinct purposes. For instance, the authors of~\cite{daunas2023empirical} consider finding a parametric data model that can best explains observed data, using $f$-divergence between the data and a reference model as inductive bias. Additionally,~\cite{cheng2020posterior} measures the $f$-divergence between true and adversarial data with the aim to improve the adversarial robustness. Moreover,~\cite{zhong2023learning} applies $f$-divergence to measure the fairness of a classifier to reduce the bias of the classifier. In contrast, our work uses $f$-divergence to measure the distance between the neural network's predictions and noisy targets. This proposed $f$-divergence regularization is designed specifically to mitigate overfitting to noisy targets, distinguishing our approach from these prior studies.
 
\section{Preliminaries}
\subsection{Data-Generating Process}
We write $\mathbf{x}\in\mathbb{R}^{d_1}$ and $\mathbf{y}\in\mathbb{R}^{d_2}$ to denote a $d_1$-dimensional spectrum and the measured oxide weights of $d_2$ targets, respectively. We
assume the following data model for $\mathbf{x}$ and $\mathbf{y}$:
\begin{align}
    \mathbf{y}=g\left(\mathbf{x}\right) +\mathbf{\epsilon}
\label{datamodel}
\end{align}
where $\mathbf{\epsilon}\in\mathbb{R}^{d_2}$ represents the realization of \textit{random} measurement error for oxide weights and $g\left(\mathbf{x}\right)$ represents the true oxide weights given a spectrum $\mathbf{x}$. 
We \textit{assume} that one has acquired $\{\left(\mathbf{x}, \mathbf{y}\right)_i\}_{i=1}^N$, whose corresponding random variable (\textit{r.v.}) pairs $\left( \mathbf{X}, \mathbf{Y}\right)_i$ are \textit{i.i.d.}, and $\left(\mathbf{X}, \mathbf{Y}\right)\sim p_{\mathbf{X}\mathbf{Y}}\left(\mathbf{x}, \mathbf{y}\right)$. 
\subsection{Graph-Based Estimation of $f$-Divergence }
$f$-divergence $D_f\left(p_0 \| p_1\right)$, proposed in~\cite{renyi1961measures}, measures the difference between two probability distributions $p_0$ and $p_1$. Formally, it is defined as follows.
\begin{definition}($f$-divergence~\cite{renyi1961measures})
Let $p_0$ and $p_1$ be two probability distributions with same support $\mathcal{S}$. Suppose $f(t)$ is a  function such that (A) $f(1)=0$ and (B) $f$ is strictly convex around $1$. Then
\begin{align}
    D_f\left(p_0 \| p_1\right)=\int_\mathcal{S}f\left(\frac{p_0(\mathbf{s})}{p_1(\mathbf{s})}\right)p_1(\mathbf{s})d\mathbf{s}
\end{align}
\end{definition}
There is a body of work~\cite{moon2014multivariate, rubenstein2019practical, kanamori2011f, berisha2014empirical, henze1988multivariate} discussing the estimation of $D_f\left(p_0 \| p_1\right)$ with samples generated from $p_0$ and $p_1$. In particular,~\cite{henze1988multivariate, berisha2014empirical} proposed a graph-based estimation method that constructs a minimum graph such as a nearest neighbor graph over the samples generated from $p_0$ and $p_1$. Then, the ratio of cut-edges, which connect samples generated from different distributions, is used to approximate $D_f\left(p_0 \| p_1\right)$. Figure~\ref{FigNNGraph} illustrates the cut-edge number of a nearest neighbor graph constructed from samples of $p_0$ and $p_1$, indicating that the cut-edge ratio is inversely proportional to the $f$-divergence between $p_0$ and $p_1$.
\begin{figure}[h!]
    \centering
\subfloat[Large divergence]{\includegraphics[width=0.23\textwidth]{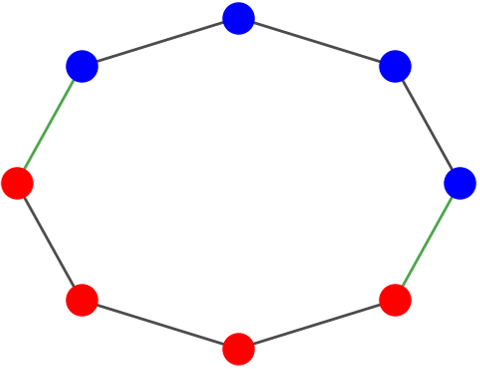}}\hfill
\subfloat[Small divergence]{\includegraphics[width=0.23\textwidth]{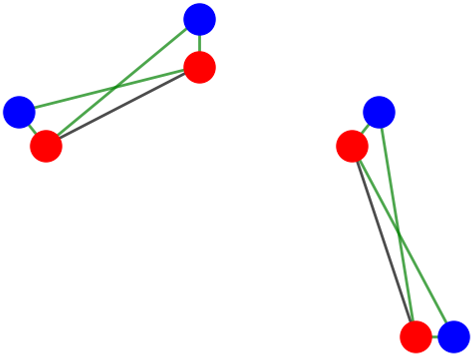}
}
\caption{(a) and (b) illustrate two scenarios of samples generated from $p_0$ and $p_1$. Samples generated from $p_0$ and $p_1$ are represented as red and blue nodes and a nearest neighbour graph is constructed over the nodes. Green edges denote the edges connecting nodes from different samples. In (a), the red nodes are farther from the blue nodes, resulting in a smaller cut-edge number, indicating a \textit{larger} $f$-divergence. In contrast, in (b), the red and blue nodes are closer, leading to a larger cut-edge number, which indicates a \textit{smaller} $f$-divergence.}
\label{FigNNGraph}
\end{figure} 
Combining Theorem 4.1 in~\cite{henze1988multivariate}, Theorem 1 in~\cite{berisha2014empirical} and Theorem 2  in~\cite{henze1999multivariate}, the asymptotic property of the cut-edge ratio is given by Theorem~\ref{Graphbasedestimation}.
\begin{theorem}(Asymptotic convergence of the cut-edge ratio) Let $n_0$ and $n_1$ denote the number of samples $\textit{i.i.d.}$ generated from $p_0$ and $p_1$, and let $T_n$ denote the cut-edge number \textit{r.v.} of the nearest neighbor graph constructed over the $n=n_0+n_1$ samples. Suppose $\lim_{n\to\infty}\frac{n_0}{n}=\alpha\in(0,1)$. Then, 
\begin{align}
  \lim_{n\to\infty}\frac{T_n}{n}=2\alpha \left(1-\alpha\right)\left(1 - D_f\left(p_0 \| p_1\right)\right)
\label{eq_dp}
\end{align}
with $f(t)=\frac{1}{4\alpha \left(1-\alpha\right)}\left(\frac{\left(\alpha t - \left(1-\alpha\right)\right)^2}{\alpha t + (1-\alpha)} - \left(2\alpha-1\right)^2\right)$. Furthermore, the resulting $D_f\left(p_0 \| p_1\right)$ satisfies: (A) $0 \leq D_f\left(p_0 \| p_1\right)\leq 1$ and (B)  $D_f\left(p_0 \| p_1\right)=D_f\left(p_1 \| p_0\right)$.
\label{Graphbasedestimation}
\end{theorem}
\begin{proof}
    Given probability density functions $p_0$ and $p_1$, let $T_n'$  denote the number of cut-edges in an~\textit{minimum spanning tree}   constructed over $n=n_0 + n_1$ points, where $n_0$ and  $n_1$ points are sampled from  $p_0$  and $n_1$ respectively.  Theorem 4.1 in~\cite{henze1988multivariate} and Theorem 1 in~\cite{henze1999multivariate} establish that $\lim_{n\to\infty}\left(\frac{T_n}{n}\right)=\lim_{n\to\infty}\left(\frac{T_n'}{n}\right)$, indicating the asymptotic ratios of cut edges in a nearest-neighbor graph and an minimum spanning tree  are equivalent.  Furthermore, Theorem 1 in~\cite{berisha2014empirical} shows that $\lim_{n\to\infty}\frac{T_n'}{n}=2\alpha\left(1-\alpha\right)\left(1 - D_f\left(p_1 \| p_0\right)\right)$ where $\alpha=\lim_{n\to\infty}\frac{n_0}{n}$. Combining these results yields $\lim_{n\to\infty}\left(\frac{T_n}{n}\right)=2\alpha\left(1-\alpha\right)\left(1 - D_f\left(p_1 \| p_0\right)\right)$. The properties of $D_f\left(p_1 \| p_0\right)$,  where  $0\leq D_f\left(p_1 \| p_0\right)\leq 1$ and $D_f\left(p_1 \| p_0\right) = D_f\left(p_0 \| p_1\right)$, are  presented in Section A in~\cite{berisha2015empirically}. 
\end{proof}
 Theorem~\ref{Graphbasedestimation} implies that, the cut-edge ratio $\frac{T_n}{n}$ of the nearest neighbor graph constructed over the samples from $p_0$ and $p_1$, converges to a quantity that can be used to recover a $f$-divergence. Moreover, the recovered $f$-divergence is bounded between 0 and 1. As readers will see in~\eqref{eq_regularizer} in Section~\ref{Methodology}, this bounded range $D_f\left(p_0 \| p_1\right) \in [0,1]$ constrains the maximum divergence allowed to be imposed between predictions and target for the neural network regularization. This greatly reduces the searching space of the regularization strength.

\section{Methodology}
\label{Methodology}
In this section, we propose a novel regularization based on $f$-divergence to prevent the training of a neural network from overfitting in multi-target regression tasks.
\subsection{Motivation}
Our data model in~\eqref{datamodel} assumes that the multi-targets $\mathbf{y}$, representing oxide weights in this work, include noise $\mathbf{\epsilon}$ arising from random factors, such as measurement errors and the highly variable conditions of the extraterrestrial environment. Training a neural network by minimizing a loss function such as mean squared error (MSE) without regularization could result in the network overfitting to the noisy multi-targets $\mathbf{y}$ in~\eqref{datamodel}. Conventional regularization methods, such as  dropout~\cite{wager2013dropout}, do not regularize the network by explicitly maintaining the distance between network predictions and noisy targets, which still poses a risk of the network overfitting to noise contained in the targets. In the following sections, we will propose an output regularization method that \textit{explicitly} maintains the divergence between the distributions of the network predictions and noisy targets during network training, thereby \textit{preventing the network from overfitting to the distribution of the noisy targets}.

\subsection{$f$-Divergence Regularization}
We write $\hat{\mathbf{y}}=\hat{g}(\mathbf{x}; \Theta)\in\mathbb{R}^{d_2}$ to denote the multi-oxide weights predictions from a neural network parameterized by $\Theta$ given a spectrum $\mathbf{x}$. Then, considering $\left(\mathbf{X}, \mathbf{Y}\right)\sim p_{\mathbf{X}\mathbf{Y}}(\mathbf{x}, \mathbf{y})$ in~\eqref{datamodel}, we write $\hat{\mathbf{Y}}=\hat{g}\left(\mathbf{X};\Theta\right)$ to denote a prediction \textit{r.v.} and recall that $\mathbf{Y}=g(\mathbf{X}) + \mathbf{\epsilon}$ in~\eqref{datamodel} represents the noisy target \textit{r.v.} As a result, we have $\hat{\mathbf{Y}}\sim p_{\hat{Y}}\left(\mathbf{\hat{y}}\right)=\int \mathbb{I}\left(\mathbf{\hat{y}}=\hat{g}\left(\mathbf{x},\Theta\right)\right)p_\mathbf{X}\left(\mathbf{x}\right)d\mathbf{x}$ and $\mathbf{Y}\sim p_{\mathbf{Y}}\left(\mathbf{y}\right)=\int p_{\mathbf{Y}\mid \mathbf{X}}\left(\mathbf{y}\mid\mathbf{x}\right)p_\mathbf{X}\left(\mathbf{x}\right)d\mathbf{x}$, where $\mathbb{I}\left(\mathbf{\hat{y}}=\hat{g}\left(\mathbf{x},\Theta\right)\right)$ is an indicator function. The presence of this indicator function in $p_{\hat{\mathbf{Y}}}\left(\mathbf{\hat{y}}\right)$ implies that  $p_{\hat{\mathbf{Y}}\mid\mathbf{X}}\left(\mathbf{\hat{y}}\mid\mathbf{x}\right)=\mathbb{I}\left(\mathbf{\hat{y}}=\hat{g}\left(\mathbf{x},\Theta\right)\right)$, meaning that $p_{\hat{\mathbf{Y}}\mid\mathbf{X}}\left(\mathbf{\hat{y}}\mid\mathbf{x}\right)$ is only non-zero at $\mathbf{\hat{y}}=\hat{g}\left(\mathbf{x},\Theta\right)$. Given that the target \textit{r.v.} $\mathbf{Y}$ contains measurement noise $\mathbf{\epsilon}$, the $f$-divergence $D_f\left(p_{\hat{\mathbf{Y}}}\| p_{\mathbf{Y}};\Theta\right)$, which is parameterized by $\Theta$ in the network  $\hat{g}\left(\mathbf{X};\Theta\right)$, can be minimized to explicitly impose a divergence between $p_{\hat{\mathbf{Y}}}$ and $p_{\mathbf{Y}}$ ,  thereby preventing the network predictions from overfitting to the distribution of the noisy targets. To this end, we introduce a loss function that includes \textit{$f$-divergence regularization} for the multi-target regression: 
\begin{align} 
\mathcal{L}\left(\Theta\right)&=\underbrace{\frac{\sum_{i=1}^{b}\sum_{j=1}^{d_2}\left(\hat{g}\left(\mathbf{x}_i;\Theta\right)_j - y_{ij}\right)^2}{bd_2}}_{\clubsuit}\nonumber\\
&+\underbrace{w \left( \hat{D}_f\left(p_{\hat{\mathbf{Y}}} \| p_{\mathbf{Y}};\Theta\right) -\gamma\right)^2}_{\spadesuit}, 
\label{eq_regularizer}
\end{align} 
where $\{\left(\mathbf{x}, \mathbf{y}\right)_i\}_{i=1}^{b},\forall (\mathbf{x},\mathbf{y})\in \mathbb{R}^{d_1}\times\mathbb{R}^{d_2}$ represents a training set of spectra and oxide weights. Here, $\clubsuit$ represents the mean squared error (MSE), and $\spadesuit$ represents the $f$-divergence regularizer. Additionally, $\hat{D}_f\left(p_{\hat{\mathbf{Y}}} \| p_{\mathbf{Y}};\Theta\right)$ (see~\eqref{eq_dp} for the $f$-function) is an empirical $f$-divergence between $p_{\mathbf{Y}}$  and $p_{\hat{\mathbf{Y}}}$ evaluated using the set of targets and predictions $\{\left(\mathbf{y}, \mathbf{\hat{y}}\right)_i\}_{i=1}^{b}$. We will discuss the calculation of $\hat{D}_f\left(p_{\hat{\mathbf{Y}}} \| p_{\mathbf{Y}};\Theta\right)$ in Section~\ref{SecImplementation}. The parameter $\gamma$ represents the enforced divergence between $p_{\hat{\mathbf{Y}}}$ and $p_{\mathbf{Y}}$ . As $\hat{D}_f$ is an approximation of $D_f\in[0,1]$ as established in Theorem~\ref{Graphbasedestimation}, $\gamma$ is also bounded between $0$ and $1$, reducing  search space for the appropriate divergence. Additionally, $w$ is the regularization strength that balances the MSE and the $f$-regularization loss. Ideally, the selection of the divergence hyperparameter $\gamma$ is proportional to the noise $\mathbf{\epsilon}$ contained in the targets $\mathbf{y}$. In practice, we can determine the hyperparameters $w$ and $\gamma$ through validation.
\subsection{$f$-Divergence Acts as an Auxiliary Loss}
As shown in~\eqref{eq_regularizer}, the $f$-divergence regularization term (denoted by $\spadesuit$) enforces larger divergences between predictions and targets when $\hat{D}_f\left(\hat{p}_{\mathbf{Y}}\|p_\mathbf{Y};\Theta\right) < \gamma$, thereby preventing the neural network's predictions from overfitting to noisy targets. Conversely, when $\hat{D}_f\left(\hat{p}_{\mathbf{Y}}\|p_\mathbf{Y};\Theta\right) > \gamma$, the $f$-divergence regularization acts as an auxiliary loss term, penalizing the network for producing predictions that are excessively different in distribution  from the targets. A pictorial description of the benefit of the $f$-divergence regularization acting as an auxiliary loss is provided in Figure~\ref{FigSelectfromCurves}. 
\begin{figure}[h!]
    \centering
\includegraphics[width=0.4\textwidth]{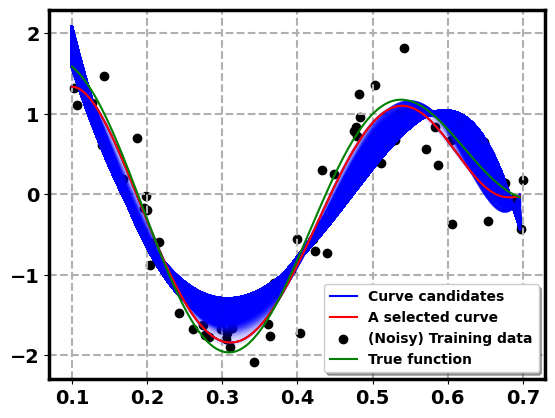}
\caption{Candidates for $L_2$ regularized curves are highlighted in blue. These curves avoid overfitting to the noisy data. By accounting for the presence of $f$-divergence between noisy data and predictions made by the target function, candidate curves with divergence exceeding a specified threshold ($\gamma$ in~\eqref{eq_regularizer}) are eliminated. This  yields a final selected curve that maintains an appropriate  $f$-divergence and closely approximates the target function.}
\label{FigSelectfromCurves}
\end{figure}
As observed,  using $L_2$ regularization independently results in many candidate curves that avoid overfitting the noisy data. However, a significant portion of these candidate curves exhibit divergences between the predictions and the noisy data larger than an appropriate threshold (i.e., $\gamma$ in~\eqref{eq_regularizer}). Combining $L_2$ regularization with the $f$-divergence regularization term can further refine the candidate curves, narrowing them down to the one that is closer to the target curve compared to independently using $L_2$, provided a suitable divergence $\gamma$ is imposed. In practice, the appropriate  $\gamma$ is selected through validation evaluation, choosing the value  with the smallest validation error. We will demonstrate the  benefits of incorporating $f$-divergence regularization into $L_2$ regularization through experiments in Section~\ref{ExpMixtures}.
\subsection{Differentiable Estimation of the $f$-Divergence for Network Training}
\label{SecImplementation}
\begin{figure}[h!]
    \centering
\subfloat[Large divergence]{\includegraphics[width=0.23\textwidth]{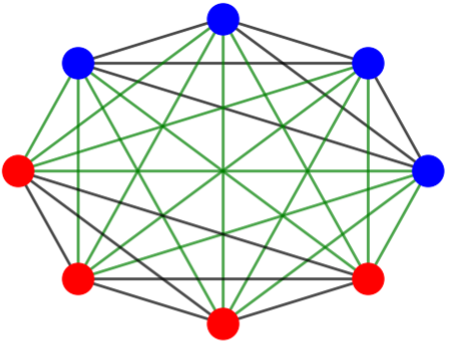}
\label{TrueFunction}}\hfill
\subfloat[Small divergence]{\includegraphics[width=0.23\textwidth]{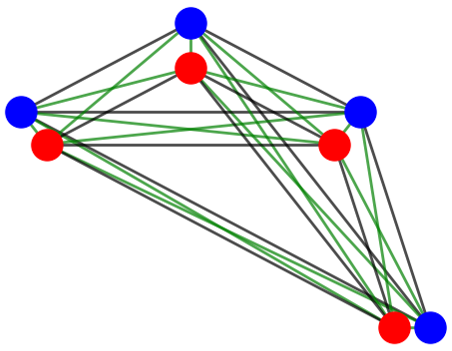}
\label{ProperDivergence}
}
\caption{(a) and (b) illustrate two scenarios of predictions (red nodes) and targets (blue nodes) in a fully connected graph, where edge weights are \textit{inversely proportional} to the distance between nodes. The green edges represent the edges between prediction and target nodes. In (a), the red nodes are farther from the blue nodes, resulting in a smaller sum of green edge weights, indicating a \textit{larger} $f$-divergence. In contrast, in (b), the red and blue nodes are closer, leading to a larger sum of green edge weights, which indicates a \textit{smaller} $f$-divergence.}
\label{FigFullyGraph}
\end{figure}
\begin{figure*}[h!]
\begin{tcolorbox}
We write $\mathcal{S}_0=\{\mathbf{y}_i\}_{i=1}^{b}$ and $\mathcal{S}_1=\{\mathbf{\hat{y}}_i\}_{i=1}^{b}$ to denote two sets of sample  realizations of \textit{r.v.s} \textit{i.i.d.} generated from $p_{\mathbf{Y}}$ and $p_{\mathbf{\hat{Y}}}$, respectively. Let $G=(V, E)$ represent a \textit{fully connected} Euclidean graph constructed over $V=\mathcal{S}_0\bigcup\mathcal{S}_1$. Additionally, let $n=|V| = |S_0| + |S_1|$ and use $\pi(v)\in \{0,1\},\forall v\in V$ to indicate the membership of $v$--whether it belongs to $\mathcal{S}_0$ or $\mathcal{S}_1$. Furthermore, $\forall v_i\in V$, we  define $\mathbf{w}_i=\left(w_{(i,1)},\cdots,w_{(i, i-1)},w_{(i, i+1)},\cdots, w_{(i,n)}\right)$ as the weights of all edges in $E$ connected to $v_i$. We also define $\sigma\left(\mathbf{w}_i\right)_j=\frac{e^{-w_{(i,j)}/\lambda}}{\sum_{u=1,u\neq i}^n e^{-w_{(i, u)}/\lambda}}$ to indicate a normalized $w_{i, j}$  after applying the softmax function $\sigma\left(\mathbf{w}_i\right)$, where $\lambda$ is a scaling parameter. Then,
\begin{align}
    \hat{t}_n= \sum_{v_i\in V}\sum_{v_j\in V}\mathbbm{1}\left(\pi\left(v_i\right)\neq \pi\left(v_j\right)\right)\sigma\left(\mathbf{w}_i\right)_j 
\label{eq_cutedge_approxition}
\end{align}
is an differentiable approximation of $t_n$.
\end{tcolorbox}
\centering
\caption{Differentiable approximation of the cut-edge ratio $t_n$~\cite{djolonga2017learning}.}
\label{MethodDifferentiable_divergence}
\end{figure*}
Theorem~\ref{Graphbasedestimation} implies that, given the targets $\{\mathbf{y}_i\}_{i=1}^{b}$ and the corresponding network predictions $\{\mathbf{\hat{y}}_i\}_{i=1}^{b}$, one can construct a nearest neighbor graph over $\{\mathbf{y}_i\}_{i=1}^{b}\bigcup \{\mathbf{\hat{y}}_i\}_{i=1}^{b}$, calculate the cut-edge number $t_n$ (a realization of $T_n$ and $n=2b$) in~\eqref{eq_dp} and recover the empirical estimate $\hat{D}_f\left(p_{\mathbf{\hat{Y}}}\| p_{\mathbf{Y}}\right)$. However, the cut-edge number $t_n$ of the nearest neighbor graph is \textit{non-differentiable} with respect to $\{\mathbf{y}_i\}_{i=1}^{b}\bigcup \{\mathbf{\hat{y}}_i\}_{i=1}^{b}$, making neural network training infeasible. The authors of~\cite{djolonga2017learning} propose a \textit{differentiable} approximation $\hat{t}_n$ of $t_n$, resulting in an differentiable estimation of $D_f\left(p_0\| p_1\right)$. Such an approximation $\hat{t}_n$ is a ``smoothed'' cut-edge number over a fully connected graph. 

Figure~\ref{FigFullyGraph} illustrates this approach, where the weights of cut-edges in a fully connected graph constructed from $\{\mathbf{y}_i\}_{i=1}^{b}$ and $\{\mathbf{\hat{y}}_i\}_{i=1}^{b}$ are summed to approximate $T_n$ as $\hat{T}_n$ in~\eqref{eq_dp}. The edge weights are \textit{inversely proportional} to the distances (e.g., Euclidean distance) between nodes, allowing $\hat{T}_n$ to capture the divergence between $\{\mathbf{y}_i\}_{i=1}^{b}$ and $\{\mathbf{\hat{y}}_i\}_{i=1}^{b}$. Specifically, a large sum $\hat{T}_n$ of cut-edge weights corresponds to  a small $f$-divergence, and vice versa. Furthermore, since the changes of predictions only affect the edge weights rather than the graph structure, $\hat{T}_n$ remains differentiable with respect to the predictions, and consequently, the parameters of a neural network. This is in stark contrast with Figure~\ref{FigNNGraph} using the cut-edge number $T_n$ in an nearest neighbor graph, in which changes in the predictions alter the graph structure, making $T_n$ in Theorem~\ref{Graphbasedestimation} non-differentiable with respect to the predictions. 

We formally present the differentiable approximation $\hat{t}_n$ in Figure~\ref{MethodDifferentiable_divergence}. \eqref{eq_cutedge_approxition} sums  the weights normalized by a softmax function over a fully-connected graph to approximate $t_n$ in~\eqref{eq_dp}, the cut-edge number in the nearest neighbor graph. 
The resulting approximation $\hat{t}_n$ is differentiable with respect to  $\{\mathbf{y}_i\}_{i=1}^{b}\bigcup \{\mathbf{\hat{y}}_i\}_{i=1}^{b}$. The $f$-divergence regularizer in~\eqref{eq_regularizer} always estimates $D_f\left(p_{\mathbf{\hat{Y}}}\| p_{\mathbf{Y}} \right)$ using network predictions and targets of the same size $b$, hence the sample ratio $\alpha$ in Theorem~\ref{Graphbasedestimation} is $0.5$. Therefore, combining~\eqref{eq_cutedge_approxition} and~\eqref{eq_dp}, a \textit{differentiable estimation of $D_f\left(p_{\mathbf{\hat{Y}}}\| p_{\mathbf{Y}} \right)$ for the $f$-divergence regularizer} is given by 
\begin{align}
   \hat{D}_f\left(p_{\mathbf{\hat{Y}}}\| p_{\mathbf{Y}} \right) = 1 - \frac{2\hat{t}_n}{n}
   \label{EqFdifferentiable}
\end{align}

\subsection{Algorithm}
In this section, we present the algorithmic implementation of the $f$-divergence regularization for neural network training in Algorithm~\ref{FdivergenceAlgo}. The algorithm takes the training and validation sets $\mathcal{D}_{\text{tr}}$ and $\mathcal{D}_{\text{val}}$, epoch number $U$, batch size $b$ and learning rate $r$ as inputs. Additionally, the input $\lambda$ represents a scale parameter that contributes to a distance metric used to construct a weighted fully-connected graph, enabling the $f$-divergence to be differentiable (See~\eqref{eq_cutedge_approxition}). The inputs  $w$ and $\gamma$ indicate the balancing factor and the imposed $f$-divergence for the regularization term  in~\eqref{eq_regularizer}. 
\begin{algorithm}[h!]
\caption{The neural network training with $f$-divergence regularization}\label{FdivergenceAlgo}
\begin{algorithmic}[1]
\STATE \textbf{Inputs: } $\mathcal{D}_{\text{tr}},\mathcal{D}_{\text{val}}, E, b, r, \lambda, w,\gamma$
\STATE \textbf{Outputs: $\hat{g}\left(\mathbf{x}; \Theta^*\right)$} 
\STATE $u\gets 0, \epsilon^* \gets \infty$ and initialize a neural network $\hat{g}\left(\mathbf{x};\Theta\right)$
\WHILE{$u<U$}
\STATE $B\gets \frac{|\mathcal{D}_{\text{tr}}|}{b},v\gets 0$
\STATE Split the training set $\mathcal{D}_{\text{tr}}$ to $B$ batches randomly 
\WHILE{$v<B$}
\STATE Calculate $\mathcal{L}\left(\Theta\right)$ with~\eqref{eq_regularizer} and~\eqref{eq_cutedge_approxition} given $\lambda$, $w$ and $\gamma$, using the $v$-th batch of $\mathcal{D}_{\text{tr}}$ 
\STATE  Update $\Theta$  with any gradient descent algorithm using the learning rate $r$
\STATE $v\gets v+1$
\ENDWHILE
\STATE Calculate the MSE $\epsilon$ with $\hat{g}\left(\mathbf{x};\Theta\right)$ using the validation set $\mathcal{D}_{\text{val}}$
\STATE \textbf{if }$\epsilon<\epsilon^*$\textbf{ then }$\Theta^*\gets \Theta$ 
\STATE $u\gets u+1$
\ENDWHILE 
\textbf{Return } $\hat{g}\left(\mathbf{x};\Theta^*\right)$ 
\end{algorithmic}
\end{algorithm}
As per the standard routine of training a neural network, we split the training set to $B$ mini-batches with a batch size of $b$ (approximately $b$, depending on the divisibility of $\left|\mathcal{D}_{\text{tr}}\right|$). A mini-batch gradient algorithm is applied to update the parameters $\Theta$ of the neural network function $\hat{g}\left(\mathbf{x};\Theta\right)$. Specifically, for each mini-batch, the predictions $\{\mathbf{\hat{y}}_{i}\}_{i=1}^b$ and corresponding targets $\{\mathbf{y}_{i}\}_{i=1}^b$ are computed,  and the $f$-divergence regularization term is evaluated using~\eqref{eq_regularizer}.
MSE of $\hat{g}\left(\mathbf{x};\Theta\right)$ is calculated on the validation set $\mathcal{D}_\text{val}$ at the end of each epoch during training. The algorithm returns $\hat{g}\left(\mathbf{x};\Theta^*\right)$ that generates the lowest validation error $\epsilon^*$. 
\section{Numerical simulations}
In this section, we present numerical simulations illustrating that minimizing an $f$-divergence-based empirical loss is more likely to find the optimal regression function compared to minimizing the mean squared error (MSE). These results numerically highlight the value of incorporating $f$-divergence into the loss function when training neural networks.
Specially, we construct a quadratic target function to generate data as follows,
\begin{align}
    Y = ax^2 + bx + E,\quad E\sim\mathcal{N}\left(0,\sigma^2\right)
\end{align}
in which $a$ and $b$ are target parameters of the quadratic function and $E$ is a Gaussian noise variable. We set $a=0.4$, $b=0.4$, and $\sigma$ as 2. Additionally, we uniformly sample $30$ points  from $\left[-2, 2\right]$ as predictors $x$, yielding a dataset $\{\left(x,y\right)_{i}\}_{i=1}^{30}$  of size $30$. Furthermore, we define a $5\times 5$ grid of candidate values for $a$ and $b$, ranging from $0.2$ to $0.6$ with intervals of $0.1$. The discrete space makes the ``grid search'' optimization feasible, where empirical risks are evaluated for each $\left(a, b\right)$ pair, and the pair with the smallest empirical risk is selected.  
\begin{figure}[h!]
    \centering
\subfloat[MSE loss]{\includegraphics[width=0.23\textwidth]{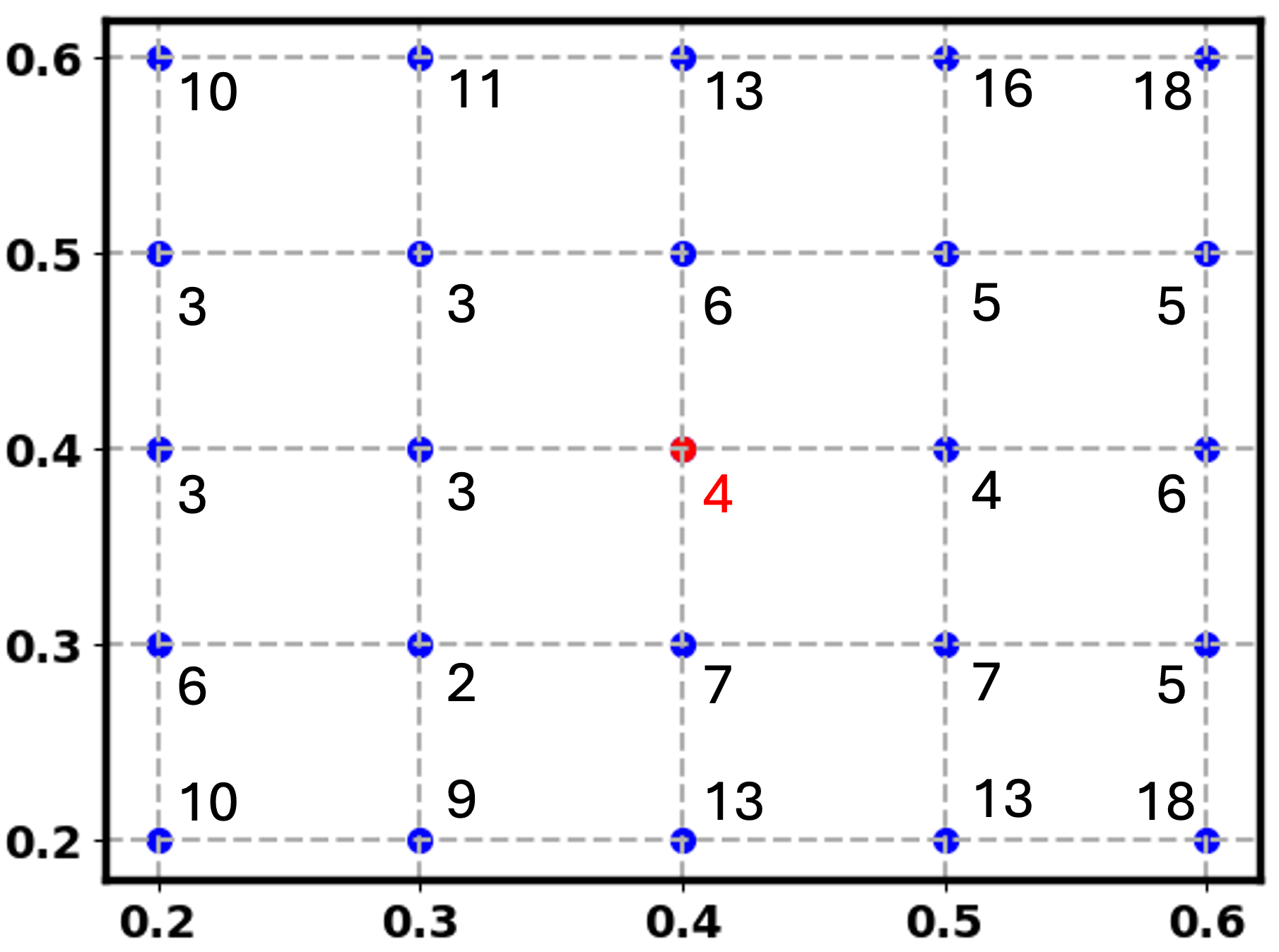}
\label{GridSpace1}}\hfill
\subfloat[$f$-divergence loss]{\includegraphics[width=0.23\textwidth]{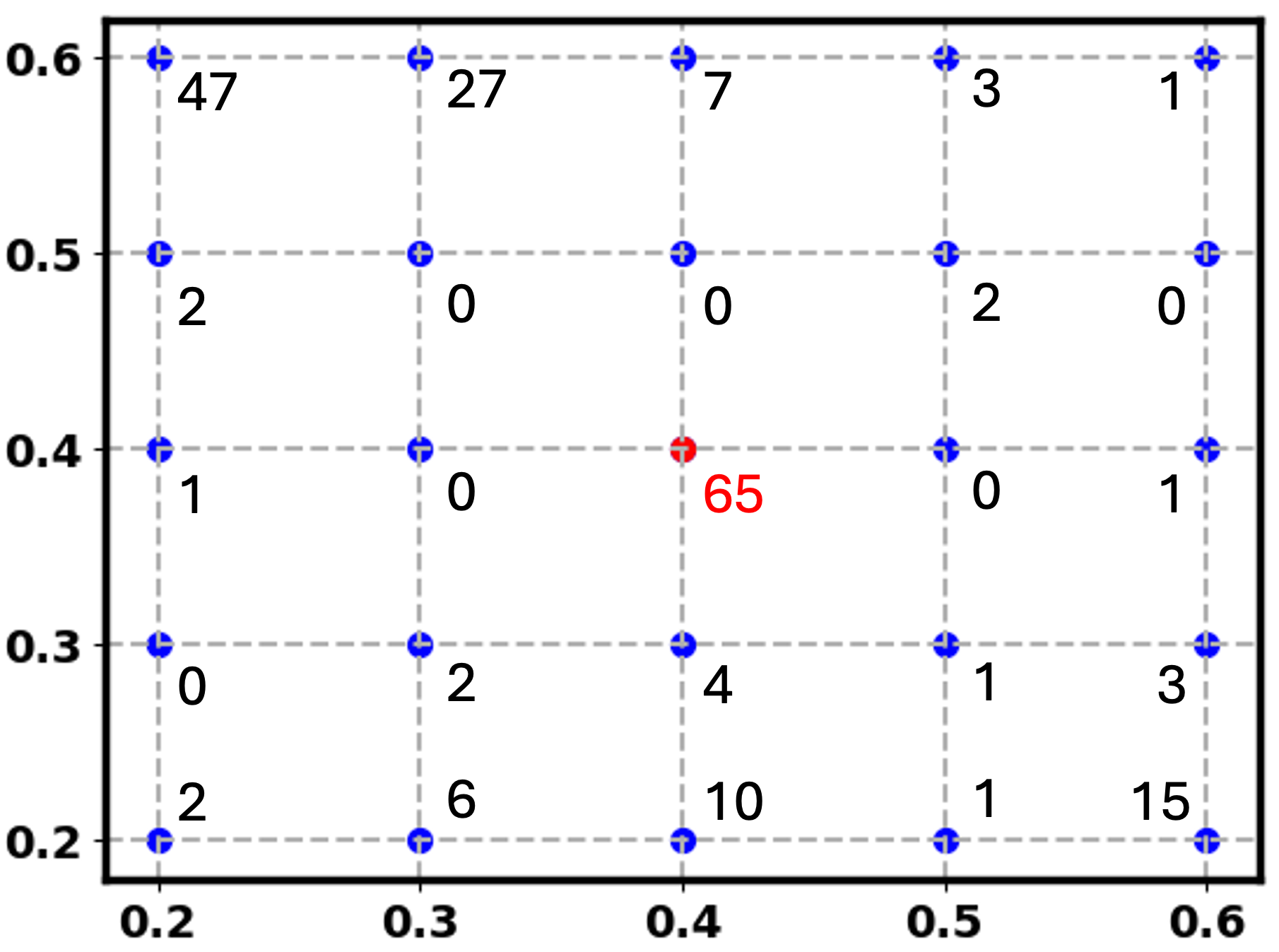}
\label{GridSpace2}
}
\caption{The frequencies of optimized parameters for minimizing the MSE and $f$-divergence loss respectively are presented in (a) and (b), overlaying a $5\times 5$ grid space of candidate parameters. The target parameter is highlighted in red. The frequency numbers are placed next to the corresponding optimized parameters. The sum of the frequencies equals to $200$ which is the number of simulation runs.}
\label{GridSpace}
\end{figure}

The task is to identify parameters $\hat{a}$ and $\hat{b}$ from the grid space to construct a quadratic function $\hat{a}x^2 +\hat{b}x$ to approximate the target $ax^2 +bx$.  We minimize two empirical risks: MSE and the $f$-divergence-based loss (only $\spadesuit$ in~\eqref{eq_regularizer}) with  $\{\left(x,y\right)_{i}\}_{i=1}^{30}$. In particular, since the ``grid search'' optimization does not require computing gradients, the non-differentiable $f$-divergence can be used to construct the $f$-divergence-based loss.
We set $\gamma=0.5$ for the $f$-divergence loss. 

The numerical simulation is repeated $200$ times, with new $\left\{(x,y)_{i}\right\}_{i=1}^n$ sampled each time,  to  obtain the distribution of minimizers $\hat{a}$ and $\hat{b}$ for both empirical risks, as shown in Figure~\ref{GridSpace}. Figure~\ref{GridSpace} illustrates the frequency of optimized parameters overlaying the $5\times 5$ grid space of parameter candidates. Notably, minimizing the empirical $f$-divergence loss hits the target $a=0.4$ and $b=0.4$ (highlighted in red)  $65$ times, significantly more often than minimizing the MSE loss, which achieves the target only $4$ times.
\\
\textbf{Why not minimize the $f$-divergence ($\spadesuit$ in~\eqref{eq_regularizer}) alone in practice?} In our numerical simulation, selecting an appropriate $f$-divergence $\gamma=0.5$ and minimizing the $f$-divergence-based empirical risk alone has a higher chance of identifying target parameters than minimizing the MSE. However, when the parameter searching space is big, such as in neural networks, \textit{it becomes necessary to construct a loss function that combines the MSE and the $f$-divergence loss, as shown in~\eqref{eq_regularizer}}. This necessity arises because evaluating the $f$-divergence between predictions and targets ignores the pairing knowledge between predictors and predictions, $\left(\mathbf{X}_i, \hat{Y}_i\right)_{i=1}^n$. As a result, it is possible to find a solution in the parameter space that generates predictions   with an appropriate $f$-divergence from the responses, even if the prediction $\hat{Y}_i$ is  far from a true target $g\left(\mathbf{X}_i\right),\forall i\in[1, n]$. An illustrative example is provided in Figure~\ref{FigPairLose}. In this case, although the $f$-divergence between the predictions and targets in Figure~\ref{FigPairLose} (b) seems appropriate, the approximation and true functions are distant. This discrepancy occurs because the evaluation of the $f$-divergence inherently loses the pairing knowledge between the predictions (or targets) and the predictors.
 \begin{figure}[h!]
    \centering
\includegraphics[width=0.4\textwidth]{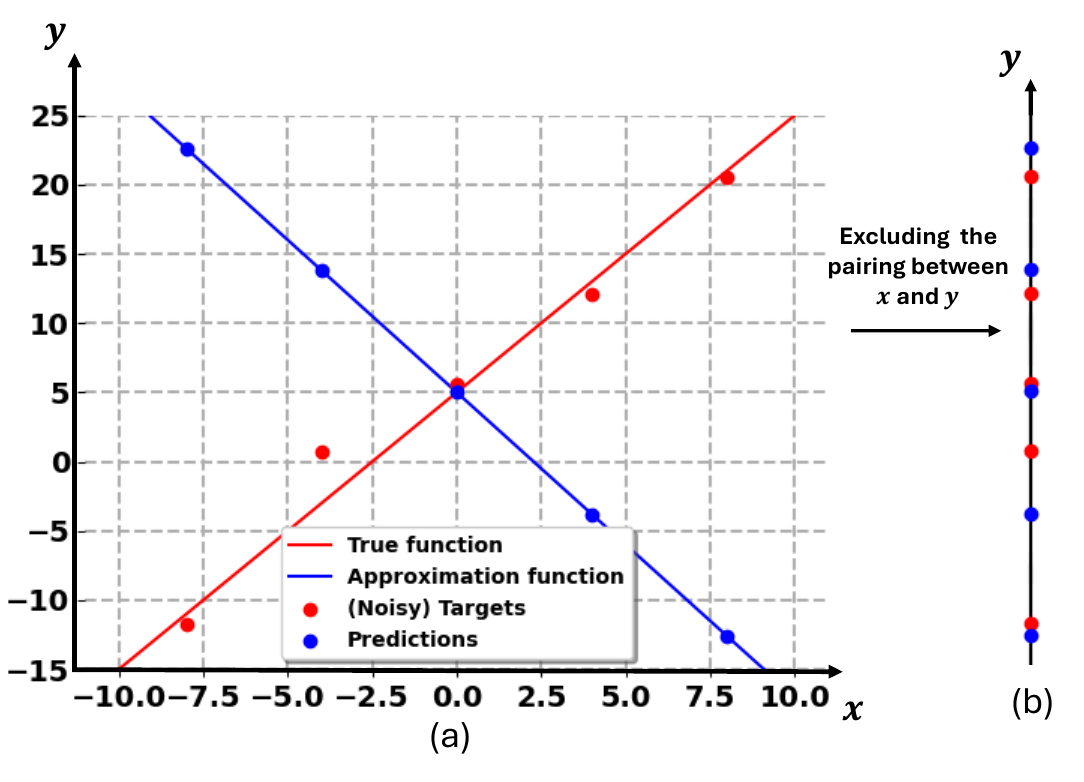}
\vspace{-0.4cm}
\caption{The $f$-divergence between predictions and targets in (b) appears appropriate even the approximation and true functions are distant, due to the loss of pairing knowledge between predictions (or targets) and predictors in (a).}
\label{FigPairLose}
\vspace{-0.4cm}
\end{figure}

\section{Experimental Results for Multi-Oxide Spectroscopic Analysis}
In this section, we present experimental results for the tasks of predicting multi-oxide weights with  Laser-Induced Breakdown Spectroscopic (LIBS) data.
\subsection{Description of the LIBS Data}
NASA has successfully landed two rovers carrying LIBS instruments,~\textit{Curiosity} in 2012~\cite{welch2013systems},  and~\textit{Perseverance} in 2020~\cite{maki2020mars}. The LIBS instruments ChemCam~\cite{maurice2012chemcam} on Curiosity and SuperCam on Perseverance~\cite{maurice2021supercam}, are central standoff (non-contact) analytical techniques of each rover's instrument suite.
SuperCam is an enhanced version of ChemCam that not only collects LIBS data over a broader spectral range but also gathers data from additional modalities, including audio.
 Equipped with the LIBS technique, both instruments emit a laser that is focused on a target several meters away. The laser ablates a small portion of the targeted sample, creating a plasma. As the plasma cools, the plasma species emit light characteristic of their electronic transitions, ionization state, and elemental composition. The resultant spectrum can be analyzed to determine elemental composition and elemental abundance for a wide variety of species.  The Los Alamos National Laboratory (LANL)  owns testbed versions of both ChemCam and SuperCam to measure the LIBS spectra for various rock samples in the simulated extraterrestrial environment. This simulated environment closely resembles real extraterrestrial conditions, such as those on Mars, for the data utilized in this work.
\begin{figure}[h!]
    \centering
\subfloat[ChemCam]{\includegraphics[width=0.23\textwidth]{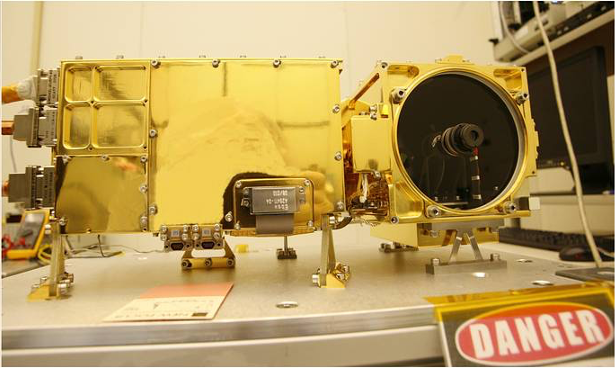}
}
\hfill
\subfloat[SuperCam]{\includegraphics[width=0.23\textwidth]{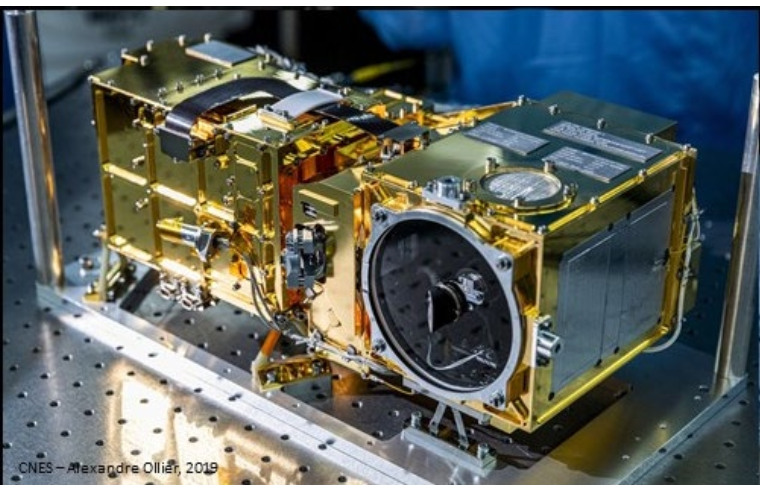}
}\\
\subfloat[A simulating chamber]{\includegraphics[width=0.23\textwidth]{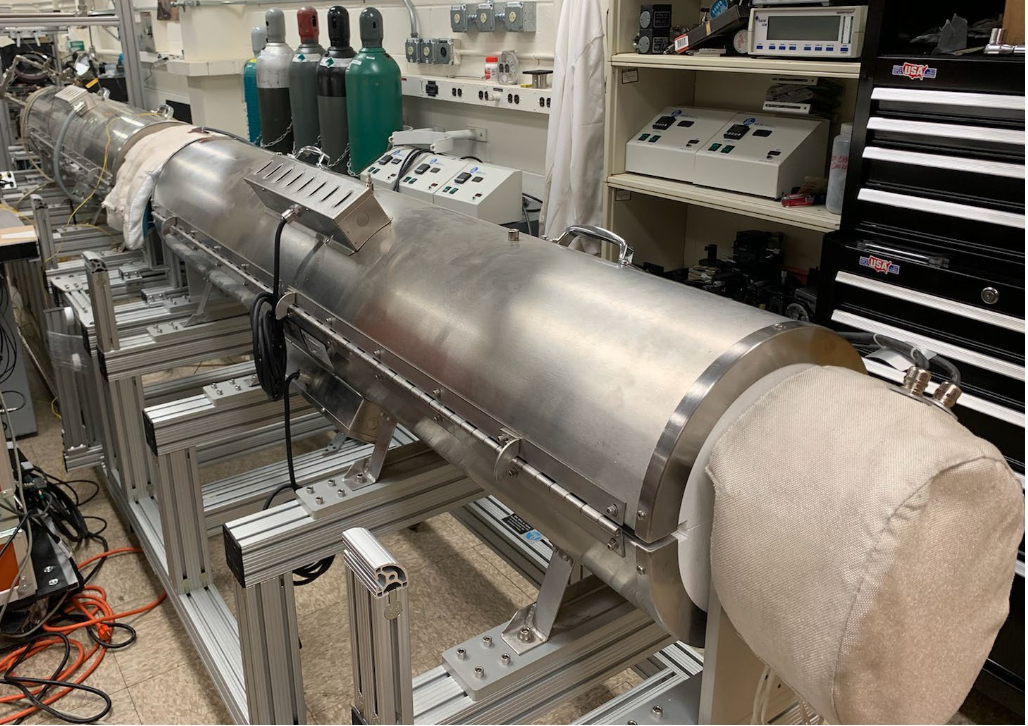}
}
\hfill
\subfloat[LIBS spectrum]{\includegraphics[width=0.23\textwidth]{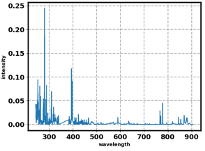}
}
\caption{The testbed ChemCam (a) and SuperCam (b), a chamber (c) for simulating an extraterrestrial environment environment and an example (d) of LIBS spectrum collected under the simulated Martian environment.}
\label{ChemCamDevice}
\end{figure}Figure~\ref{ChemCamDevice} (a), (b) and (c) describe the testbed ChemCam and SuperCam, and the chamber for simulating the extraterrestrial environment. Additionally, Figure~\ref{ChemCamDevice} (d) provides an example of the LIBS spectrum measured from a rock sample in a simulated Martian environment. Consequently, scientists at LANL have used the testbed  ChemCam and SuperCam to collect~\textit{two standard datasets} for the task of spectroscopic analysis using LIBS data. The \textbf{ChemCam} dataset is comprised of the LIBS measurements for  $\sim$550 rock samples in the simulated Martian environment, along with nine oxide-weights of these samples.  The \textbf{SuperCam} dataset is comprised of the LIBS measurements for $\sim$300 rock samples in the simulated Martian environment, along with eight oxide-weights of these samples. These data are publicly available in the planetary data system (PDS)~\cite{mcmahon1996overview}. It is anticipated that, by using the datasets prepared in advance, one can develop an automation tool based on ML trained with these data to perform the spectroscopic analysis for the multi-oxide weight prediction for the rocks collected on Mars. 
\subsection{The task, network architecture, training details and evaluation methods}
\textbf{The task:} The primary task in this work is to predict weights of multiple oxide within rock samples using LIBS spectroscopic data. With significant advancements in computational power and the proliferation of industrial data, CNNs have achieved widespread success in automating complex tasks such as autonomous driving and fraud detection in finance. However, the datasets for ChemCam and SuperCam are relatively small, typically containing only hundreds of samples—several orders of magnitude smaller than industrial datasets. Consequently, applying CNNs to the task of predicting multi-oxide weights presents a higher risk of overfitting compared to industrial applications.  To address this challenge, we propose training a 1D-CNN with an innovative $f$-divergence regularization, as detailed in Algorithm~\ref{FdivergenceAlgo}, to achieve a CNN solution with smaller prediction error for the multi-response prediction.
\begin{table*}
\centering
\caption{Details of the hyper-parameter sets for the various regularization methods} 
\resizebox{1\textwidth}{!}{
\begin{tabular}{c|cccccccccccc} \hline
 $L_1$ strengths &0.00005&0.0001&0.0003&0.0005&0.0007&0.001\\\hline
$L_2$ strengths &0.00005&0.0001&0.0003&0.0005&0.0007&0.001\\\hline
Dropout probabilities &0.03&0.04&0.05&0.06&0.07&0.08&0.1\\\hline
$\left(w, \gamma\right)$ for the $f$-divergence  &(0.00005, 0.001)&(0.00007, 0.0012)&(0.00009, 0.0014) &(0.0001, 0.02)& (0.00012, 0.005)&(0.00014, 0.01)&(0.00016, 0.015)&(0.0003, 0.015)&(0.0005, 0.02)\\
regularization~\eqref{eq_regularizer}&(0.0007, 0.025)&(0.005, 0.03)&(0.01, 0.02)
\\\hline
\end{tabular}
}
\label{TableHyperParameters}
\end{table*}

\textbf{The network architecture:} We employ a 1D-CNN with three convolutional blocks, each comprising a convolutional layer, batch normalization, ReLU activation, and max-pooling for dimensionality reduction. The network begins with a single input channel and progressively reduces the number of output channels from 32 to 16 and finally to 8. These convolutional blocks are followed by a fully connected network that maps the outputs of the final block to multiple end nodes, representing the predicted multi-oxide weights.

\textbf{The training details and evaluation methods:}  
As will be  detailed in Section~\ref{secexpanalysis}, we evaluate the performance of the 1D-CNN trained under various conditions: without regularization, and with $L_1$, $L_2$, dropout and the proposed $f$-divergence regularization. To ensure a fair comparison, we maintain an identical training setup across all experiments. Specifically, the ChemCam/SuperCam datasets are split into training, validation, and test sets with an 80/10/10 ratio. We set the training parameters as follows: the number of training epochs $U = 500$, batch size $b = 16$, learning rate $r = 1$, and scale parameter $\lambda = 2$. The scale parameter is used solely for calculating the differentiable $f$-divergence in~\eqref{eq_cutedge_approxition}. Network training was performed using the Adadelta optimizer~\cite{zeiler2012adadelta}. Training with regularization requires selecting hyperparameters, such as regularization strengths, which are detailed in Table~\ref{TableHyperParameters}. The validation set is used for hyperparameter tuning, and performance is evaluated on the test set. Moreover, we repeat the training routines across 15 runs and report the average~\textit{root mean squared error (RMSE)} on the test set. Finally, we perform a t-test to assess whether the differences in RMSE between our method and the compared methods are statistically significant.
\begin{table}
\centering
\caption{RMSE comparison between $f$-divergence and various methods, including no regularization, $L_1$, $L_2$ and dropout methods. Smaller RMSE values are highlighted. Paired t-test results indicate statistical significance: \cmark\text{ } indicates $f$-divergence has significantly smaller RMSE, \textbf{-} indicates no significant difference, and \xmark\text{ } indicate the compared method has significantly smaller RMSE. A significant level $0.1$ is used for the t-test.}
\textbf{ChemCam results}\vspace{0.1cm}
\resizebox{0.49\textwidth}{!}{
\begin{tabular}{c|cccccccccc} \hline
 \backslashbox{Method}{Oxide}&All Oxides&SiO$_2$&TiO$_2$&Al$_2$O$_3$&FeO&MnO&MgO&CaO&Na$_2$O&K$_2$O\\\midrule
  No regularization &2.33&4.15&0.62&\textbf{1.84}&2.54&1.19&1.05&1.40&\textbf{0.56}&\textbf{0.64}\\
$f$-divergence (proposed) &\textbf{2.25}&\textbf{4.07}&\textbf{0.59}&1.89&\textbf{2.24}&\textbf{1.13}&\textbf{0.97}&\textbf{1.38}&\textbf{0.56}&0.65
\\\hline
$L_1$ &2.35&4.31&0.57&\textbf{1.83}&2.33&\textbf{1.04}&\textbf{0.93}&1.59&0.60&0.66\\
$f$-divergence (proposed) &\textbf{2.25}&\textbf{4.07}&\textbf{0.59}&1.89&\textbf{2.24}&1.13&0.97&\textbf{1.38}&\textbf{0.56}&\textbf{0.65}\\\hline
$L_2$&2.30&4.12&0.60&1.91&2.40&1.18&1.03&\textbf{1.33}&0.58&\textbf{0.63}\\
$f$-divergence (proposed) &\textbf{2.25}&\textbf{4.07}&\textbf{0.59}&\textbf{1.89}&\textbf{2.24}&\textbf{1.13}&\textbf{0.97}&1.38&\textbf{0.56}&0.65\\\hline
Dropout&2.33&4.20&\textbf{0.56}&1.97&2.36&\textbf{1.11}&1.00&1.41&0.59&0.73\\
$f$-divergence (proposed) &\textbf{2.25}&\textbf{4.07}&0.59&\textbf{1.89}&\textbf{2.24}&1.13&\textbf{0.97}&\textbf{1.38}&\textbf{0.56}&\textbf{0.65}
\\\hline
\end{tabular}
}\vspace{0.2cm}
\resizebox{0.49\textwidth}{!}{
\begin{tabular}{c|cccccccccc} \hline
 \backslashbox{Method}{Oxide}&All Oxides&SiO$_2$&TiO$_2$&Al$_2$O$_3$&FeO&MnO&MgO&CaO&Na$_2$O&K$_2$O\\\midrule
No reg. V.S. Ours&-&-&\cmark&-&\cmark&\cmark&\cmark&-&-&-\\\hline
$L_1$ V.S. Ours &\cmark&-&-&-&-&\xmark&\xmark&\cmark&\cmark&-\\\hline
$L_2$ V.S. Ours&-&-&-&-&\cmark&-&-&-&-&\xmark\\\hline
Dropout V.S. Ours&\cmark&-&\xmark&\cmark&\cmark&-&-&-&\cmark&\cmark\\\hline
\end{tabular}
}\vspace{0.2cm}
\textbf{SuperCam results}
\resizebox{0.49\textwidth}{!}{
\begin{tabular}{c|ccccccccc}\hline 
\backslashbox{Method}{Oxide}&All Oxides&SiO$_2$&TiO$_2$&Al$_2$O$_3$&FeO&MgO&CaO&Na$_2$O&K$_2$O\\\midrule
No regularization&2.67&\textbf{4.80}&0.69&2.42&\textbf{2.58}&1.30&\textbf{1.33}&0.71&\textbf{0.82}\\
$f$-divergence (proposed) &\textbf{2.66}&4.90&\textbf{0.61}&\textbf{2.27}&2.70&\textbf{1.11}&1.40&\textbf{0.65}&0.84
\\\hline
$L_1$ &\textbf{2.65}&4.91&\textbf{0.60}&\textbf{2.25}&\textbf{2.66}&1.28&\textbf{1.32}&0.69&\textbf{0.83}\\
$f$-divergence (proposed) &2.66&\textbf{4.90}&0.61&2.27&2.70&\textbf{1.11}&1.40&\textbf{0.65}&0.84\\\hline
$L_2$&\textbf{2.66}&4.96&\textbf{0.55}&\textbf{2.24}&\textbf{2.50}&1.16&\textbf{1.33}&\textbf{0.65}&0.86\\
$f$-divergence (proposed) &\textbf{2.66}&\textbf{4.90}&0.61&2.27&2.70&\textbf{1.11}&1.40&\textbf{0.65}&\textbf{0.84}\\\hline
Dropout&2.69&\textbf{4.83}&\textbf{0.59}&2.34&2.71&1.37&\textbf{1.40}&0.68&\textbf{0.84}\\
$f$-divergence (proposed) &\textbf{2.66}&4.90&0.61&\textbf{2.27}&\textbf{2.70}&\textbf{1.11}&\textbf{1.40}&\textbf{0.65}&\textbf{0.84}
\\\hline
\end{tabular}
}\vspace{0.2cm}
\resizebox{0.49\textwidth}{!}{
\begin{tabular}{c|ccccccccc}\hline 
\backslashbox{Method}{Oxide}&All Oxides&SiO$_2$&TiO$_2$&Al$_2$O$_3$&FeO&MgO&CaO&Na$_2$O&K$_2$O\\\midrule
No reg. V.S. Ours &-&-&\cmark&\cmark&-&\cmark&-&\cmark&-\\\hline
$L_1$ V.S. Ours  &-&-&-&-&-&\cmark&-&\cmark&-\\\hline
$L_2$ V.S. Ours&-&-&\xmark&-&-&\cmark&-&-&-\\\hline
Dropout V.S. Ours&-&-&-&-&-&\cmark&-&\cmark&-
\\\hline
\end{tabular}
}
\label{TableChemCamAllReg}
\end{table}
\subsection{Quantitative comparisons between $f$-divergence regularization and various regularization methods}
\label{secexpanalysis}
This section presents the quantitative evaluations of various methods, including training a network without any regularization, as well as with $L_1$, $L_2$, dropout~\cite{srivastava2014dropout} and the proposed $f$-divergence regularization. Table~\ref{TableChemCamAllReg} summarizes these evaluations, comparing the RMSEs of $f$-divergence with those of other methods. Furthermore, Table~\ref{TableChemCamAllReg} includes t-test results to indicate whether the observed performance improvements are statistically significant.

\begin{table*}[h!]
\centering
\caption{RMSE comparison between standard regularization and standard regularization combined with $f$-divergence regularization across various regularization strengths. Smaller RMSE values are highlighted. Paired t-test results indicate statistical significance: \cmark\text{ } indicates that the combined method has significantly smaller RMSE, \textbf{-} indicates no significant difference, and \xmark\text{ } indicates that the independent use of standard regularization has significantly smaller RMSE. A significant level $0.1$ is used for the t-test.} 
\textbf{ChemCam results}\hspace{7.4cm}\textbf{SuperCam results}\\
\resizebox{0.51\textwidth}{!}{
\begin{tabular}{c|cccccccccc} \hline
 \backslashbox{Method}{Oxide}&All Oxides&SiO$_2$&TiO$_2$&Al$_2$O$_3$&FeO&MnO&MgO&CaO&Na$_2$O&K$_2$O\\\midrule
$L_1$ with strength $0.0001$ &\textbf{2.30}&\textbf{4.09}&\textbf{0.62}&1.94&\textbf{2.38}&1.18&1.07&\textbf{1.39}&0.64&\textbf{0.64}\\
Combined (Ours)   &2.36&4.46&0.67&\textbf{1.85}&2.42&\textbf{1.00}&\textbf{0.92}&1.40&\textbf{0.57}&0.65\\\hline
$L_1$ with strength $0.0003$ &2.43&4.42&\textbf{0.61}&1.96&2.38&1.26&\textbf{1.03}&1.54&\textbf{0.58}&\textbf{0.65}\\
Combined (Ours)  &\textbf{2.35}&\textbf{4.30}&0.64&\textbf{1.89}&\textbf{2.23}&\textbf{0.96}&1.25&\textbf{1.48}&0.64&\textbf{0.65}\\\hline
$L_1$ with strength $0.0005$ &2.64&5.05&\textbf{0.61}&2.07&\textbf{2.54}&\textbf{1.18}&\textbf{1.19}&1.54&\textbf{0.55}&0.68\\
Combined (Ours) &\textbf{2.52}&\textbf{4.55}&0.65&\textbf{2.05}&2.66&1.23&1.22&\textbf{1.45}&0.67&\textbf{0.67}\\\hline
$L_1$ with strength $0.0007$ &2.71&5.10&\textbf{0.67}&2.21&2.60&1.29&1.28&\textbf{1.64}&0.67&0.83\\
Combined (Ours)  &\textbf{2.63}&\textbf{4.91}&0.71&\textbf{2.04}&\textbf{2.56}&\textbf{1.17}&\textbf{1.15}&\textbf{1.64}&\textbf{0.62}&\textbf{0.73}
\\\hline
\end{tabular}
}
\hfill 
\resizebox{0.47\textwidth}{!}{
\begin{tabular}{c|ccccccccc} \hline
 \backslashbox{Method}{Oxide}&All Oxides&SiO$_2$&TiO$_2$&Al$_2$O$_3$&FeO&MgO&CaO&Na$_2$O&K$_2$O\\\midrule
$L_1$ with strength $0.0001$ &2.68&4.98&\textbf{0.58}&2.28&2.66&\textbf{1.19}&1.36&0.68&0.87\\
Combined (Ours)  &\textbf{2.54}&\textbf{4.72}&0.66&\textbf{2.25}&\textbf{2.39}&1.22&\textbf{1.29}&\textbf{0.64}&\textbf{0.79}\\\hline
$L_1$ with strength $0.0003$ &2.69&\textbf{4.74}&0.66&2.33&2.71&1.37&1.52&0.68&0.86\\
Combined (Ours)&\textbf{2.59}&4.77&\textbf{0.65}&\textbf{2.19}&\textbf{2.52}&\textbf{1.28}&\textbf{1.25}&\textbf{0.66}&\textbf{0.82}\\\hline
$L_1$ with strength $0.0005$ &2.83&5.40&\textbf{0.68}&\textbf{2.33}&\textbf{2.74}&\textbf{1.31}&\textbf{1.43}&\textbf{0.73}&\textbf{0.78}\\
Combined (Ours)  &\textbf{2.75}&\textbf{4.97}&\textbf{0.68}&2.40&2.80&1.35&1.50&0.74&0.83\\\hline
$L_1$ with strength $0.0007$ &2.86&\textbf{5.06}&\textbf{0.65}&2.41&2.88&1.61&1.59&0.75&0.84\\
Combined (Ours) &\textbf{2.76}&5.22&0.71&\textbf{2.10}&\textbf{2.69}&\textbf{1.37}&\textbf{1.46}&\textbf{0.67}&\textbf{0.78}
\\\hline
\end{tabular}
}\vspace{0.2cm}
\resizebox{0.51\textwidth}{!}{
\begin{tabular}{c|cccccccccc} \hline
 \backslashbox{Method}{Oxide}&All Oxides&SiO$_2$&TiO$_2$&Al$_2$O$_3$&FeO&MnO&MgO&CaO&Na$_2$O&K$_2$O\\\midrule
$L_1$ with strength $0.0001$ V.S. Combined (Ours)&\xmark&\xmark&\xmark&\cmark&-&\cmark&\cmark&-&\cmark&-\\\hline
$L_1$ with strength $0.0003$ V.S. Combined (Ours)&\cmark&-&\xmark&-&\cmark&\cmark&\xmark&-&\xmark&-\\\hline
$L_1$ with strength $0.0005$ V.S. Combined (Ours) &\cmark&\cmark&\xmark&-&\xmark&-&-&\cmark&\xmark&-\\\hline
$L_1$ with strength $0.0007$ V.S. Combined (Ours) &\cmark&-&\xmark&\cmark&-&\cmark&\cmark&-&\cmark&\cmark\\\hline
\end{tabular}
}
\hfill
\resizebox{0.47\textwidth}{!}{
\begin{tabular}{c|ccccccccc} \hline
 \backslashbox{Method}{Oxide}&All Oxides&SiO$_2$&TiO$_2$&Al$_2$O$_3$&FeO&MgO&CaO&Na$_2$O&K$_2$O\\\midrule
$L_1$ with strength $0.0001$ V.S. Combined (Ours) &\cmark&\cmark&\xmark&-&\cmark&-&\cmark&\cmark&\cmark\\\hline
$L_1$ with strength $0.0003$ V.S. Combined (Ours) &\cmark&-&-&\cmark&\cmark&\cmark&\cmark&-&-\\\hline
$L_1$ with strength $0.0005$ V.S. Combined (Ours)&\cmark&\cmark&-&-&-&-&\xmark&-&\xmark\\\hline
$L_1$ with strength $0.0007$ V.S. Combined (Ours)&\cmark&-&\xmark&\cmark&\cmark&\cmark&\cmark&\cmark&\cmark\\\hline
\end{tabular}
}\\
\vspace{0.1cm}
(a) Comparison between $L_1$ and $L_1$ combined with $f$-divergence regularization for various $L_1$ strengths.\\
\vspace{0.4cm}
\resizebox{0.51\textwidth}{!}{
\begin{tabular}{c|cccccccccc} \hline
 \backslashbox{Method}{Oxide}&All Oxides&SiO$_2$&TiO$_2$&Al$_2$O$_3$&FeO&MnO&MgO&CaO&Na$_2$O&K$_2$O\\\midrule
$L_2$ with strength $0.0001$ &2.34&4.30&0.65&1.87&2.47&1.19&1.07&1.35&0.61&\textbf{0.58}\\
Combined (Ours)   &\textbf{2.25}&\textbf{4.12}&\textbf{0.64}&\textbf{1.82}&\textbf{2.32}&\textbf{1.12}&\textbf{1.00}&\textbf{1.30}&\textbf{0.58}&0.64\\
\hline
$L_2$ with strength $0.0003$ &2.33&\textbf{4.24}&\textbf{0.64}&1.97&\textbf{2.36}&1.11&1.03&1.37&\textbf{0.55}&0.67\\
Combined (Ours)  &\textbf{2.30}&4.26&0.65&\textbf{1.87}&2.41&\textbf{1.07}&\textbf{0.97}&\textbf{1.26}&0.55&\textbf{0.62}\\
\hline
$L_2$ with strength $0.0005$ &\textbf{2.26}&\textbf{4.03}&0.61&\textbf{1.88}&2.33&1.20&1.12&1.26&\textbf{0.57}&\textbf{0.60}\\
Combined (Ours)  &2.26&4.15&\textbf{0.60}&1.90&\textbf{2.29}&\textbf{1.20}&\textbf{0.96}&\textbf{1.29}&0.59&0.66\\
\hline
$L_2$ with strength $0.0007$ &2.36&4.29&0.64&2.00&2.37&1.17&\textbf{0.98}&1.41&\textbf{0.58}&0.70\\
Combined (Ours)  &\textbf{2.26}&\textbf{4.09}&\textbf{0.60}&\textbf{1.77}&\textbf{2.43}&\textbf{1.17}&1.04&\textbf{1.34}&0.58&\textbf{0.57}\\\hline
\end{tabular}
}
\hfill 
\resizebox{0.47\textwidth}{!}{
\begin{tabular}{c|ccccccccc} \hline
 \backslashbox{Method}{Oxide}&All Oxides&SiO$_2$&TiO$_2$&Al$_2$O$_3$&FeO&MgO&CaO&Na$_2$O&K$_2$O\\\midrule
$L_2$ with strength $0.0001$ &2.62&\textbf{4.67}&0.68&2.35&\textbf{2.65}&1.19&\textbf{1.37}&\textbf{0.62}&\textbf{0.80}\\
Combined (Ours)   &\textbf{2.59}&4.69&\textbf{0.59}&\textbf{2.17}&2.66&\textbf{1.10}&1.38&0.67&0.85\\\hline
$L_2$ with strength $0.0003$ &2.69&4.93&\textbf{0.63}&2.27&2.77&\textbf{1.25}&\textbf{1.31}&\textbf{0.67}&0.85\\
Combined (Ours)  &\textbf{2.58}&\textbf{4.79}&0.65&\textbf{2.18}&\textbf{2.49}&\textbf{1.25}&1.32&0.73&\textbf{0.81}\\\hline
$L_2$ with strength $0.0005$ &2.62&\textbf{4.63}&\textbf{0.62}&\textbf{2.22}&2.71&1.26&1.37&0.71&0.88\\
Combined (Ours)  &\textbf{2.59}&4.65&0.67&2.38&\textbf{2.54}&\textbf{1.18}&\textbf{1.34}&\textbf{0.65}&\textbf{0.82}\\\hline
$L_2$ with strength $0.0007$ &2.72&4.88&0.61&2.25&2.79&1.34&1.35&0.66&0.87\\
Combined (Ours)  &\textbf{2.50}&\textbf{4.47}&\textbf{0.55}&\textbf{2.22}&\textbf{2.41}&\textbf{1.25}&\textbf{1.32}&\textbf{0.62}&\textbf{0.84}
\\\hline
\end{tabular}
}\vspace{0.2cm}
\resizebox{0.51\textwidth}{!}{
\begin{tabular}{c|cccccccccc} \hline
 \backslashbox{Method}{Oxide}&All Oxides&SiO$_2$&TiO$_2$&Al$_2$O$_3$&FeO&MnO&MgO&CaO&Na$_2$O&K$_2$O\\\midrule
$L_2$ with strength $0.0001$ V.S. Combined (Ours)&\cmark&\cmark&-&-&\cmark&\cmark&\cmark&-&-&\xmark\\\hline
$L_2$ with strength $0.0003$ V.S. Combined (Ours)&-&-&-&\cmark&-&-&\cmark&\cmark&-&\cmark\\\hline
$L_2$ with strength $0.0005$ V.S. Combined (Ours) &-&-&-&-&-&-&\cmark&-&-&\xmark\\\hline
$L_2$ with strength $0.0007$ V.S. Combined (Ours) &\cmark&\cmark&-&\cmark&-&-&\xmark&\cmark&\xmark&\cmark\\\hline
\end{tabular}
}
\hfill
\resizebox{0.47\textwidth}{!}{
\begin{tabular}{c|ccccccccc} \hline
 \backslashbox{Method}{Oxide}&All Oxides&SiO$_2$&TiO$_2$&Al$_2$O$_3$&FeO&MgO&CaO&Na$_2$O&K$_2$O\\\midrule
$L_2$ with strength $0.0001$ V.S. Combined (Ours) &-&-&\cmark&\cmark&-&\cmark&-&\xmark&\xmark\\\hline
$L_2$ with strength $0.0003$ V.S. Combined (Ours) &\cmark&-&-&-&\cmark&-&-&\xmark&-\\\hline
$L_2$ with strength $0.0005$ V.S. Combined (Ours)&-&-&\xmark&\xmark&\cmark&\cmark&-&\cmark&\cmark\\\hline
$L_2$ with strength $0.0007$ V.S. Combined (Ours)&\cmark&\cmark&\cmark&-&\cmark&\cmark&-&-&-\\\hline
\end{tabular}
}\\
\vspace{0.1cm}
(b) Comparison between $L_2$ and $L_2$ combined with $f$-divergence regularization for various $L_2$ strengths. \\
\vspace{0.4cm}
\resizebox{0.51\textwidth}{!}{
\begin{tabular}{c|cccccccccc} \hline
 \backslashbox{Method}{Oxide}&All Oxides&SiO$_2$&TiO$_2$&Al$_2$O$_3$&FeO&MnO&MgO&CaO&Na$_2$O&K$_2$O\\\midrule
Dropout with rate $0.04$ & 2.34 & \textbf{4.17} & \textbf{0.58} & 1.94 & 2.47 & \textbf{1.19} & \textbf{1.03} & 1.39 & 0.60 & \textbf{0.63} \\
Combined (Ours) & \textbf{2.30} & 4.21 & 0.62 & \textbf{1.85} & \textbf{2.32} & 1.23 & 1.15 & \textbf{1.29} & \textbf{0.59} & 0.68 \\
\hline
Dropout with rate $0.06$ & 2.47 & 4.54 & \textbf{0.61} & \textbf{1.99} & 2.60 & \textbf{1.19} & \textbf{1.04} & 1.43 & \textbf{0.56} & 0.73 \\
Combined (Ours) & \textbf{2.43} & \textbf{4.35} & 0.63 & 2.10 & \textbf{2.25} & 1.26 & 1.23 & \textbf{1.52} & 0.61 & \textbf{0.67} \\
\hline
Dropout with rate $0.08$ & \textbf{2.58} & 4.57 & 0.68 & 2.21 & \textbf{2.54} & \textbf{1.21} & \textbf{1.33} & \textbf{1.61} & 0.71 & \textbf{0.70} \\
Combined (Ours) & 2.60 & \textbf{4.40} & \textbf{0.65} & \textbf{2.16} & 2.73 & 1.26 & 1.34 & 1.72 & \textbf{0.67} & 0.75 \\
\hline
Dropout with rate $0.1$ & 2.83 & 4.87 & 0.67 & 2.46 & 2.75 & \textbf{1.39} & 1.59 & 1.84 & 0.75 & \textbf{0.76} \\
Combined (Ours) & \textbf{2.61} & \textbf{4.55} & \textbf{0.63} & \textbf{2.07} & \textbf{2.48} & 1.40 & \textbf{1.33} & \textbf{1.80} & \textbf{0.69} & 0.79
\\\hline
\end{tabular}
}
\hfill 
\resizebox{0.47\textwidth}{!}{
\begin{tabular}{c|ccccccccc} \hline
 \backslashbox{Method}{Oxide}&All Oxides&SiO$_2$&TiO$_2$&Al$_2$O$_3$&FeO&MgO&CaO&Na$_2$O&K$_2$O\\\midrule
Dropout with rate $0.04$ & 2.59 & 4.64 & 0.64 & \textbf{2.16} & 2.58 & 1.39 & \textbf{1.47} & 0.67 & \textbf{0.82} \\
Combined (Ours) & \textbf{2.57} & \textbf{4.61} & \textbf{0.54} & 2.21 & \textbf{2.48} & \textbf{1.26} & 1.52 & \textbf{0.65} & \textbf{0.82} \\
\hline
Dropout with rate $0.06$ & 2.71 & 4.73 & 0.74 & 2.44 & 2.52 & 1.41 & 1.56 & \textbf{0.69} & 0.86 \\
Combined (Ours) & \textbf{2.62} & \textbf{4.71} & \textbf{0.63} & \textbf{2.21} & \textbf{2.49} & \textbf{1.20} & \textbf{1.50} & 0.77 & \textbf{0.84} \\
\hline
Dropout with rate $0.08$ & 2.84 & 4.93 & \textbf{0.64} & 2.71 & \textbf{2.65} & 1.46 & \textbf{1.58} & 0.70 & \textbf{0.86} \\
Combined (Ours) & \textbf{2.79} & \textbf{4.74} & 0.65 & \textbf{2.69} & 2.66& \textbf{1.41} & 1.71 & \textbf{0.64} & 0.87 \\
\hline
Dropout with rate $0.1$ & \textbf{2.88} & \textbf{4.91} & 0.65 & 2.63 & \textbf{2.54} & \textbf{1.58} & \textbf{1.81} & 0.88 & \textbf{0.83} \\
Combined (Ours) & 2.89 & 4.99 & \textbf{0.60} & \textbf{2.44} & 2.71 & 1.65 & 1.91 & \textbf{0.83} & 0.90
\\\hline
\end{tabular}
}\vspace{0.2cm}
\resizebox{0.51\textwidth}{!}{
\begin{tabular}{c|cccccccccc} \hline
 \backslashbox{Method}{Oxide}&All Oxides&SiO$_2$&TiO$_2$&Al$_2$O$_3$&FeO&MnO&MgO&CaO&Na$_2$O&K$_2$O\\\midrule
Dropout with rate $0.04$ V.S. Combined (Ours)&\cmark&-&\xmark&\cmark&\cmark&-&\xmark&\cmark&-&\xmark\\\hline
Dropout with rate $0.06$ V.S. Combined (Ours)&\cmark&\cmark&-&\xmark&\cmark&\xmark&\xmark&\xmark&\xmark&\cmark\\\hline
Dropout with rate $0.08$ V.S. Combined (Ours) &-&\cmark&\cmark&-&\xmark&\xmark&-&\xmark&\cmark&\xmark\\\hline
Dropout with rate $0.1$ V.S. Combined (Ours) &\cmark&\cmark&\cmark&\cmark&\cmark&-&\cmark&-&\cmark&-\\\hline
\end{tabular}
}
\hfill
\resizebox{0.47\textwidth}{!}{
\begin{tabular}{c|ccccccccc} \hline
 \backslashbox{Method}{Oxide}&All Oxides&SiO$_2$&TiO$_2$&Al$_2$O$_3$&FeO&MgO&CaO&Na$_2$O&K$_2$O\\\midrule
Dropout with rate $0.04$ V.S. Combined (Ours) &-&-&\cmark&-&-&\cmark&-&-&-\\\hline
Dropout with rate $0.06$ V.S. Combined (Ours) &\cmark&-&\cmark&\cmark&-&\cmark&-&\xmark&-\\\hline
Dropout with rate $0.08$ V.S. Combined (Ours)&-&\cmark&-&-&-&-&\xmark&\cmark&-\\\hline
Dropout with rate $0.1$ V.S. Combined (Ours)&-&-&\cmark&\cmark&\xmark&-&\xmark&\cmark&\xmark\\\hline
\end{tabular}
}\\
\vspace{0.1cm}
(c) Comparison between dropout and dropout combined with $f$-divergence regularization for various dropout rates.
\label{TableRegMixture}
\end{table*}
In the evaluation using ChemCam data, when considering the~\textit{predictive performance averaged across all nine oxides}, the $f$-divergence regularization achieves a smaller RMSE (reported under the ``all oxides'') compared to no-regularization, $L_1$, $L_2$ and dropout. Additionally, at least~\textit{six out of the nine oxide-weight predictions} using $f$-divergence regularization exhibit lower RMSE. We have also conducted paired t-test to compare the mean differences in RMSE between $f$-regularization and other methods. These comparisons were made using pairs of prediction results for each rock sample in the test set. A significance level of $0.1$ was used for the t-tests. The t-test results indicate that $f$-divergence regularization has significantly smaller RMSE than $L_1$ and dropout, when considering the \textit{average predictive performance across all nine oxide weights}. Furthermore, the t-test results reveal that $f$-divergence produces significantly smaller RMSEs  than the compared methods, particularly  no regularization and dropout, in multiple~\textit{predictions of individual oxide weights}.

In the evaluation using SuperCam data, the~\textit{predictive performances averaged across all nine-oxide weights}, reported by RMSE under ``all oxides'', are comparable between $f$-divergence regularization and the other methods. T-test results also indicate no significant differences in these RMSE comparisons are found when considering the predictive performance averaged across all oxide weights. In contrast, $f$-divergence regularization has produced substantially smaller RMSEs in multiple~\textit{predictions of individual oxide weights} compared with other methods. Our t-test results show that these RMSEs achieved by $f$-divergence regularization are significantly smaller than those of no regularization, $L_1$ and dropout methods. In summary, $f$-divergence regularization is generally better or not worse than the compared methods for most comparisons.

\subsection{Quantitative evaluations on the combination of $f$-divergence and standard regularization methods}
\label{ExpMixtures}
In this section, we evaluate the performance of combining standard regularization methods, including $L_1$, $L_2$, and dropout, with the proposed $f$-divergence regularization using the Chemcam and SuperCam datasets.  We also compare these results with the independent use of standard regularization methods. Table~\ref{TableRegMixture} presents this comparison across various regularization strength groups: $[0.0001, 0.0003, 0.0005, 0.0007]$ for $L_1$ and $L_2$, and dropout rates of $[0.04, 0.06, 0.08, 0.1]$. A significance level of $0.1$ is used for generating the t-test results.

Table~\ref{TableRegMixture} (a) presents an RMSE comparison between $L_1$ and the combination of $L_1$ with $f$-divergence regularization, demonstrating that the combined method achieves smaller RMSE values than the independent use of $L_1$ in most comparisons for the \textit{predictive performance averaged across all oxides}. Additionally, t-test results confirm the statistical significance of these improvements. Notably, the combined method also produces comparable or significantly smaller RMSE values for \textit{predictions of individual oxide weights}, particularly at $L_1$ strengths of $0.0001$ and $0.0007$. This is further supported by t-test results, which indicate that combining $L_1$ with $f$-divergence regularization yields significantly smaller RMSE than the independent use of $L_1$ across most comparisons within the various $L_1$ strength groups for the multi-oxide weight predictions.

Table~\ref{TableRegMixture}(b) and (c) present comparisons similar to Table~\ref{TableRegMixture}(a), however the standard regularization $L_1$ is replaced with $L_2$ and dropout regularization. In these comparisons, we have observed that, \textit{in the predictive performance averaged across all oxides}, combining $L_2$ and dropout with $f$-divergence regularization consistently produces smaller RMSE than the independent use of $L_2$ and dropout, respectively. The t-test results also demonstrate that these RMSE values by the combined methods are either comparable to or significantly smaller than those of the independent use of $L_2$ and dropout. In addition to the consistent predictive performance averaged across all oxides, combining $L_2$ and dropout with $f$-divergence regularization also produces smaller or comparable RMSE results in the \textit{predictions of individual oxide weights} in multiple comparisons with the independent use of $L_2$ and dropout, respectively. These results are pronounced in the $L_2$ regularization comparison using ChemCam and SuperCam, and the dropout comparison using SuperCam. T-test results indicate that the combined method produces significantly smaller RMSE than the independent use of $L_2$ and dropout in most of these comparisons.

In summary, Table~\ref{TableRegMixture} demonstrates that combining the standard regularization methods with $f$-divergence regularization improves or maintains the predictive performance averaged across all oxides in most comparisons with the independent use of the standard regularization methods. Additionally, in the prediction of a single oxide-weight, the combined method outperforms the independent use of the standard regularization using the SuperCam data in the prediction of weights for multiple oxides, and presents comparable performance using the ChemCam data.
\subsection{Ablation Study}
We observed in Table~\ref{TableRegMixture} for ChemCam results that combining standard regularization with $f$-divergence regularization produces higher RMSE than the independent use of standard regularization in \textit{some comparisons} for predicting weights of multiple oxides. This effect is particularly pronounced in the ChemCam experiment with a dropout rate of $0.06$, where t-test results show that independently using dropout leads to significantly better predictions for five out of nine oxides compared to the combined method.

To investigate these underperforming results of the combined method, we conducted an ablation study. We hypothesize that these results are caused by using an MSE averaged across all oxide-weight errors as the primary loss function to train and validate the neural network. Consequently, the combined method appears more consistent in the predictive performance averaged across all oxides than in predicting the weight of a single oxide. To test this hypothesis, we trained neural networks with an MSE calculated only for the weight prediction error of a single oxide. We selected the oxides Al$_2$O$_3$, MnO, MgO, CaO, and Na$_2$O, training five neural networks with either dropout or the combined method (combining $f$-divergence regularization with dropout), using a dropout rate of $0.06$, to predict their respective oxide weights.
\begin{table}[h]
\centering
\caption{RMSE comparison between dropout and the combined method with a dropout rate of $0.06$ for single-oxide weight prediction networks. The paired t-test results are also provided: \cmark\text{ } denotes that the combined method has significantly smaller RMSE, \textbf{-} indicates no significant difference, and \xmark\text{ } signifies that the independent use of dropout has significantly smaller RMSE. A significance level of $0.1$ is used for the t-test.}
\resizebox{0.3\textwidth}{!}{
\begin{tabular}{c|ccccc} \hline
 \backslashbox{Method}{Oxide}&Al$_2$O$_3$&MnO&MgO&CaO&Na$_2$O\\\midrule
Dropout with rate $0.06$&\textbf{1.59}&0.91&\textbf{0.88}&1.15&0.55\\
Combined (Ours)&\textbf{1.59}&\textbf{0.76}&\textbf{0.88}&\textbf{1.14}&\textbf{0.52}
\\\hline
\end{tabular}
}
\resizebox{0.18\textwidth}{!}{
\begin{tabular}{ccccc} \hline
Al$_2$O$_3$&MnO&MgO&CaO&Na$_2$O\\\midrule
-&\cmark&-&-&\cmark
\\\hline
\end{tabular}
}
\label{ExpAblationStudy}
\end{table}
As shown in Table~\ref{TableRegMixture} (c) for the ChemCam results, the combined method underperforms compared to the independent use of dropout with a dropout rate of $0.06$ when an MSE averaged across all oxides is used as the primary loss function during training. However, the results from Table~\ref{ExpAblationStudy} for our ablation study reveal the opposite: the combined method produces significantly smaller RMSE than dropout for MnO and Na$_2$O and delivers comparable performance for the other oxides. These findings suggest that if the goal is to predict the weight of a specific oxide, it is more effective to train a single-oxide weight prediction network using the combined method rather than training a multi-oxide weight prediction network.

\section{Conclusion}
This paper addresses the critical challenge of performing spectroscopic analysis to characterize the oxide composition of rock samples collected from Martian-like environments. Accurate chemical characterization of these rock samples is essential for various downstream planetary science tasks, such as uncovering astro-biological activities. To address the dichotomy between the data-hungry nature of convolutional neural networks (CNNs) and the limited availability of spectroscopic data, we propose an innovative regularization method based on $f$-divergence. This novel regularization explicitly enforces a divergence constraint between model predictions and targets, effectively mitigating overfitting. Our experimental results demonstrate that incorporating $f$-divergence regularization into CNN training provides significant benefits compared to standard regularization methods, such as $L_1$, $L_2$, and dropout. More importantly, $f$-regularization outperforms these methods in prediction accuracy particularly when averaging across all oxide weights in experiments using ChemCam and SuperCam data. This enhanced performance could enable precise mineralogical profiling of rock samples, which demands high-accuracy characterization of multiple oxide compositions simultaneously.
\section*{Acknowledgments}
Presented research was supported by the Laboratory Directed Research and Development program of Los Alamos National Laboratory under project number 20240065DR.
\bibliography{Reference}
\bibliographystyle{IEEEtran}

\newpage


\end{document}